%% file: arxiv.tex
\documentclass[preprint,12pt]{clear2025} % Include author names

\title[Nontransitive Paradoxes in Causality]{Omitted Labels Induce Nontransitive Paradoxes in Causality}
\usepackage{times}
\usepackage[utf8]{inputenc} % allow utf-8 input
\usepackage[T1]{fontenc}    % use 8-bit T1 fonts
\usepackage{hyperref}       % hyperlinks
\usepackage{url}            % simple URL typesetting
\usepackage{booktabs}       % professional-quality tables
\usepackage{amsfonts}       % blackboard math symbols
\usepackage{nicefrac}       % compact symbols for 1/2, etc.
\usepackage{microtype}      % microtypography
\usepackage{xcolor} % colors
\input{packages}
\input{shortcuts}

\clearauthor{%
 \Name{Bijan Mazaheri}\thanks{Work completed while a PhD student at California Institute of Technology. Partially supported by an Amazon AI4Science Fellowhip.} \Email{bijan.h.mazaheri@dartmouth.edu}\\
 \addr Dartmouth Engineering
 \AND
 \Name{Siddharth Jain}\thanks{Now at Ripple Medical Inc., Irvine, CA. Work completed while a postdoc at California Institute of Technology.} \Email{sid496@gmail.com}\\
 \addr California Institute of Technology%
 \AND
 \Name{Matthew Cook} \Email{m.cook@rug.nl}\\
 \addr University of Groningen%
 \AND
 \Name{Jehoshua Bruck} \Email{bruck@paradise.caltech.edu}\\
 \addr California Institute of Technology%
}

\begin{document}

\maketitle

\begin{abstract}
We explore ``omitted label contexts,'' in which training data is limited to a subset of the possible labels. This setting is standard among specialized human experts or specific focused studies. By studying Simpson's paradox, we observe that ``correct'' adjustments sometimes require non-exchangeable treatment and control groups. A generalization of Simpson's paradox leads us to study networks of conclusions drawn from different contexts, within which a paradox of nontransitivity arises. We prove that the space of possible nontransitive structures in these networks exactly corresponds to structures that form from aggregating ranked-choice votes.
\end{abstract}

\section{Introduction}
Knowledge is powered and limited by the data that drives it. When seeking to understand the relevance of a study, the most critical aspect is its data's \emph{context}. Two common ``data contexts'' are (1) the population of participants and (2) interventions made on that population. Optimizing a context for utility might involve a census of the target population with perfectly focused interventions. The real world, however, weighs utility against feasibility; budget constraints limit a study's participants, and moral constraints limit its interventions.

\paragraph{Domain Adaptation}
A study may attempt to transfer conclusions from a sub-optimal \emph{data context} to a \emph{target} one. Most broadly, this topic is known as domain adaptation (DA). DA formally deals with transferring from a training probability distribution $p(\vec{V})$ to a target distribution $q(\vec{V})$. For a specific training task, we often separate the measured variables into $\vec{V} = (Y, \vec{X})$, where $Y$ is the ``label'' and $\vec{X}$ are the ``covariates.''

DA is made possible through assumptions on the shared information between $p$ and $q$. Two such assumptions are \emph{covariate shift} \citep{shimodaira2000improving} and \emph{label shift} \citep{schweikert2008empirical}. 
Another setting that can be considered to fall within DA is causal inference, which seeks to identify the effects of intervening on a variable (often a treatment) without actually performing that intervention \citep{pearl2009causality, imbens2015causal, peters2017elements}.  From a DA perspective, our data-context is ``observational'' and our target context is ``interventional,'' i.e. $q(\vec{V}) = p(\vec{V} \given \Do (T = t))$ for $T \in \vec{V}$.

The best way to learn causal effects is to perform interventions in a clinical ``randomized controlled trial'' (RCT). RCTs utilize a random partition into treatment and control groups, which ensures that no potentially relevant factors (like the severity of a disease) can affect the probability of receiving treatment. This ``exchangeability'' between treatment and control forms a key principle of the interventional distributions that causal inference aims to emulate. That is, an intervened variable's value (such as $\Do(T=t)$) contains no information about that variable's usual causes.

A common technique within DA is to transform the data's distribution to the target context through reweighting. These weights make up for poorly represented portions of a data distribution by giving them more importance. For label shift, this corresponds to weighting data by the ratio of label probabilities ($w(\vec{x}, y) = q(y) / p(y)$). Similarly, covariate shift can be accounted for using a ratio of covariate probability densities ($w(\vec{x}, y) = q(\vec{x}) / p(\vec{x})$). Causal inference performs similar reweightings to transform observational data into distributions with exchangeable treatment and control groups, e.g. ``inverse probability weighting'' \citep{imbens2015causal, cole2008constructing, hernan2020whatif}. Covariate adjustment techniques like the ``backdoor adjustment'' \citep{pearl2009causality} and ``G-computation'' \citep{robins2009longitudinal} perform adjustments at the distribution level that can nonetheless be thought of as reweighting data. 

\paragraph{Omitted Label Contexts}
This paper will discuss a relatively new branch of the DA tree introduced by \citet{mazaheri2021synthesizing}, which we will call ``omitted label contexts.'' Such settings are limited to \emph{only a subset} $\mathcal{Y}^* \subset \mathcal{Y}$ of the labels for $Y \in \mathcal{Y}$. For example, ``dogs vs. cats'' is omitted label context (omitting all other animals),  but ``dogs vs. non-dogs'' is not. While the relative probabilities of classes within this subset are maintained, data from all other labels are unobserved. More precisely, $p(y^*_1)/ p(y^*_2) = q(y^*_1)/ q(y^*_2)$ for $y*_1, y*_2 \in \mathcal{Y}^*$, but $p(y') = 0$ if $y' \not \in \mathcal{Y^*}$. Within the scope of this paper, we will restrict our focus to $\abs{\mathcal{Y}^*} = 2$.

Omitted label contexts are motivated by a few real-life scenarios within medicine and epidemiology. Omitted label contexts are extremely common in the study of rare conditions. For example, a census genome sequencing of the US population would be an impractical and financially infeasible task. Instead, databases like TCGA \citep{tomczak2015review} allow focused access to patients with specific (and often rare) cancers by omitting the ``healthy label''. In study designs, investigators may opt for an omitted label context (e.g. breast vs lung cancer) or induce further label shift by working with a uniform distribution on the labels of interest. In these cases, it is useful to determine a notion of consistency. Though less related to the issues discussed in this paper, ``immortal time bias'' \citep{suissa2008immortal} can also be thought of as an omitted label context in which patients are excluded due to death before the end of the study.

Most broadly, omitted labels are a form of sampling bias --- a topic that has been studied in detail within the causal inference literature \citep{correa2019identification}. \citet{bareinboim2015recovering} calls a causal effect ``recoverable'' if it can be computed in the presence of a selection mechanism.
An important difficulty within omitted label contexts is that they are ``irreversible.'' Zero-probability labels cannot be ``weighted-up'' to transform the distribution to that of the general population. With respect to covariate adjustments, this makes the causal effects unrecoverable. 

\subsection{Summary}
This paper will begin by illustrating the potentially severe consequences of omitted label bias, particularly with respect to unrecoverable causal effects. We consider what happens when these effects are ignored, demonstrating that omitted labels can give rise to reversals in calculated treatment effects. This paradox is driven by the non-exchangeability of ``correct'' causal inference weights in omitted label contexts, i.e., standard adjustments to exchangeability give us incorrect quantities.

\paragraph{Contributions}
The problems that arise within omitted label contexts are unsurprising --- computing (what is known to be) an unrecoverable causal effect is impossible. However, within the details of these failures emerges a structure between the conclusions of multiple studies with \emph{different} restrictions on their labels. Specifically, a ``new'' \footnote{The paradox is very similar to the one introduced in \citet{mazaheri2021synthesizing} but has not been observed for causality of covarite shift.} paradox manifests within the combination of erroneous causal conclusions. We will link this paradox to a 200-year-old observation from social choice theory known as the ``Condorcet paradox,'' which demonstrates how ranked-choice votes (i.e., a preference ordering of the candidates for each voter) can result in a cycle of aggregate preference \citep{nicolas1785essai}. In a surprising connection between social choice theory, causality, and machine learning, we prove that these phenomena are one and the same.

The task of combining conclusions from different models is sometimes referred to as decision fusion \citep{castanedo2013review}. The usual reasons for studying decision fusion involve lack of access to data or inability to combine datasets (e.g., because of non-overlapping covariates). As advances in large models (such as LLMs) transition us from data-based to model-based vehicles of information, studying how to fuse decisions from multiple models will be of growing importance.

\paragraph{Outline} Section~\ref{sec: Simpsons Paradox} begins by explaining Simpson's paradox through an example that will be expanded on throughout the paper. Section~\ref{sec: omitted label contexts} studies omitted label contexts and shows how and why they can exhibit Simpson's paradox in causality. Section~\ref{sec:networks of contexts} introduces Condorcet's paradox and the linear ordering polytope and shows their relationship to networks of conclusions from different omitted label contexts. Section~\ref{sec:discussion} concludes the paper and discusses its implications.

\paragraph{Notation}
\newcommand{\onenorm}[1]{\norm{#1}_1}
\newcommand{\indicator}[1]{\mathbbm{1}[#1]}
\newcommand{\indvec}[1]{\vec{1_{\ell}}[#1]}
\newcommand{\simplex}{\triangle}
\newcommand{\sit}{G^{\vec{P}}_u}
\newcommand{\reg}{G^{\vec{P}}_d}
\renewcommand{\Classes}{\mathcal{Y}}
In general, we will use the capital Latin alphabet (i.e., $X, Y, T$) to denote random variables, with $Y$ being the ``label'' or class we wish to predict or determine the causal effect on, $X$ being the covariates, and $T$ being the treatment. The lowercase Latin alphabet will denote assignments to these variables, e.g., $x^{(1)}$ means $X = x^{(1)}$. Vectors and sets of random variables will be in bold-face font, while other types of sets will use caligraphic font (e.g., $\mathcal{Y} = \{y^{(1)}, y^{(2)}, \ldots \}$).  The following notation is used throughout the paper:
\begin{itemize}
    \itemsep0em 
    \item $\Pr(\cdot)$ will be used for probability. $p(\cdot)$ and $q(\cdot)$ will also be used when discussing DA.
    \item $\onevec$ denotes an all $1$ vector of size $\ell$.
    \item $\triangle^\ell$ denotes the $\ell$-dimensional simplex. That is, $\lambda \in \triangle^\ell$ iff $\lambda \in [0, 1]^\ell$ and $\vec{1}^\top \lambda = 1$.
    \item $\prec$ and $\preceq$ denote element-wise inequality. For example, we say $\vec{w} \preceq \vec{v}$ if $w_i\leq v_i ~\forall~ i\in[\ell].$
    \item We will use $\co(S)$ to denote the open convex hull of $S$, $\cco(S)$ to denote the closed convex hull, and $\bo(\cdot)$ to denote the boundary.
\end{itemize}

\section{Simpson's Paradox} \label{sec: Simpsons Paradox}
We will ease into our dissection of errors in causal quantities by discussing Simpson's paradox. While these observations are well-studied, we present them here in a form that is generalized to multiple labels in the following section. This discussion will rely on hypothetical observational data on a treatment $T$ and its outcome $Y$, given in Table~\ref{tab: simpsons paradox} (left). To map these quantities to a real-world example, suppose $X$ indicates the severity of a disease. Severe cases receive treatment at higher rates (and recover at lower rates) than mild ones. In aggregate, treated patients have lower rates of improvement than untreated patients.

\begin{table}
\centering
        \scalebox{.8}{\begin{tabular}{|c|c|c|c|}
        \hline
        $T$ & $X$ & $y^{(0)}$ & $y^{(1)}$ \\
        \hline
        $t^{(0)}$ & $x^{(0)}$ & $3$ & $7$ \\
        \hline
        $t^{(0)}$ & $x^{(1)}$ & $1$ & $0$ \\
        \hline
        $t^{(1)}$ & $x^{(0)}$ & $0$ & $1$ \\
        \hline
        $t^{(1)}$ & $x^{(1)}$ & $7$ & $3$ \\
        \hline
\end{tabular}}\hspace{1cm}\scalebox{.8}{\begin{tabular}{|c|c|c|c|}
        \hline
        $T$ & $X$ & $y^{(0)}$ & $y^{(1)}$ \\
        \hline
        $t^{(0)}$ & $x^{(0)}$ & $3$ & $7$ \\
        \hline
        $t^{(0)}$ & $x^{(1)}$ & $\mathbf{10}$ & $\mathbf{0}$ \\
        \hline
        $t^{(1)}$ & $x^{(0)}$ & $\mathbf{0}$ & $\mathbf{10}$ \\
        \hline
        $t^{(1)}$ & $x^{(1)}$ & $7$ & $3$ \\
        \hline
\end{tabular}}  
\caption{The left table shows a specification of counts for Simpson's Paradox. On the right, the 2nd and 3rd rows have been weighted up so that both the treated ($t^{(1)}$) and untreated ($t^{(0)}$) groups have a (weighted) distribution on $X$ that matches the marginal distribution, which is uniform (all rows must sum to $10$).}
\label{tab: simpsons paradox}
\end{table}
\begin{table}
\centering
            \scalebox{.8}{\begin{tabular}{|c|c|c|c|c|}
        \hline
        $T$ & $X$ & $y^{(0)}$ & $y^{(1)}$ & $y^{(2)}$\\
        \hline
        $t^{(0)}$ & $x^{(0)}$ & $3$ & $7$ & $0$ \\
        \hline
        $t^{(0)}$ & $x^{(1)}$ & $1$ & $0$ & $99$ \\
        \hline
        $t^{(1)}$ & $x^{(0)}$ & $0$ & $1$ & $99$ \\
        \hline
        $t^{(1)}$ & $x^{(1)}$ & $7$ & $3$ & $0$\\
        \hline
\end{tabular}}\hspace{1cm}\scalebox{.8}{\begin{tabular}{|c|c|c|c|c|}
        \hline
        $T$ & $X$ & $y^{(0)}$ & $y^{(1)}$ & $y^{(2)}$\\
        \hline
        $t^{(0)}$ & $x^{(0)}$ & $\mathbf{30}$ & $\mathbf{70}$ & $\mathbf{0}$ \\
        \hline
        $t^{(0)}$ & $x^{(1)}$ & $1$ & $0$ & $99$ \\
        \hline
        $t^{(1)}$ & $x^{(0)}$ & $0$ & $1$ & $99$ \\
        \hline
        $t^{(1)}$ & $x^{(1)}$ & $\mathbf{70}$ & $\mathbf{30}$ & $\mathbf{0}$\\
        \hline
\end{tabular}}
            \caption{On the left we give an augmentation of Table~\ref{tab: simpsons paradox} with a third column. Now, the 1st and 4th rows must be weighted up in order to balance the distribution of $X$ (all rows must sum to $100$).}
            \label{tab: augmented simpsons}
\end{table}
\begin{table}
\centering
            \scalebox{.8}{\begin{tabular}{|c|c|c|c|c|}
        \hline
        $T$ & $X$ & $y^{(0)}$ & $y^{(1)}$ & $y^{(2)}$\\
        \hline
        $t^{(1)}$ & $x^{(0)}$ & $2$ & $ 1$ & $0$\\
        \hline
        $t^{(1)}$ & $x^{(1)}$ & $0$ & $2$ & $1$ \\
        \hline
        $t^{(1)}$ & $x^{(2)}$ & $1$ & $0$ & $2$ \\
        \hline
        $t^{(0)}$ & $x^{(0)}$ & $0$ & $1$ & $2$\\
        \hline
        $t^{(0)}$ & $x^{(1)}$ & $2$ & $0$ & $1$ \\
        \hline
        $t^{(0)}$ & $x^{(2)}$ & $1$ & $2$ & $0$ \\
        \hline
\end{tabular}}
            \caption{A specification of counts that mimics Condorcet's paradox.}
            \label{tab: condorcet paradox}
\end{table}

The driver for Simpson's paradox is the difference in severity between those who did and did not receive treatment and the effect of this severity on patient outcomes. That is to say, the treatment and control groups are not exchangeable. To remedy this, we can reweight the rows of our table to exchangeability by emphasizing the severely ill patients who did not receive treatment and the mildly ill patients who did receive treatment. This is accomplished by reweighting datapoints $(t, x, y)$ according to the inverse probability/propensity of receiving the treatment that they got, $w(t, x, y) = 1/\Pr(t \given x)$, sometimes referred to as ``Inverse Probability Weighting'' (IPW)  \citep{imbens2015causal}. In our example, IPW weights up the second and third rows by a factor of $10$, as shown in Table~\ref{tab: simpsons paradox} (right). When this reweighting is interpreted as a synthetic study on $40$ participants ($20$ treated and $20$ control, each with a $10:10$ split on severity), the new apparent treatment effect is $\nicefrac{13}{20} - \nicefrac{7}{20} = 30\%$.

An alternative perspective is that the causal effect of the treatment lies in the outcome changes \emph{within each severity group}. By separately considering the severe and mild patients, we can average outcomes according to the marginal probability distribution of severity. Following this intuition, the ``backdoor adjustment'' \citep{pearl2009causality} calculates the probability distribution of $Y=y^{(i)}$ under an intervention of $T=t^{(j)}$:
\begin{equation*}
    \Pr(y^{(i)} \given \Do(t^{(j)})) \coloneq \sum_x \Pr(x) \Pr(y^{(i)} \given x, t^{(j)}).
\end{equation*}

The difference between the two possible interventions gives the ``average treatment effect'' (ATE),
\begin{equation*}
    \text{ATE} = \Pr(y^{(1)} \given \Do(t^{(1)})) - \Pr(y^{(1)} \given \Do(t^{(0)})) = \frac{\nicefrac{1}{1} + \nicefrac{3}{10}}{2} - \frac{\nicefrac{7}{10} + \nicefrac{0}{3}}{2} = .3.
\end{equation*}
Notice that the marginal probability distribution of $X$ is uniform, corresponding to an equal weighting of the $x^{(0)}$ and $x^{(1)}$ rows in Table~\ref{tab: simpsons paradox}. In fact, both IPW and backdoor approaches result in the same weightings of the rows of the table because $\Pr(t, x)/\Pr(t \given x) = \Pr(x)$.

Simpson's paradox has been the subject of a long list of works to which it would be impossible to do full justice. \citet{pearl2022comment} and \citet{hernanSimpsons} describe Simpson's paradox as ``solved'' by causal modeling. We will focus on one key takeaway: the choice of how to re-weight sub-cases (rows of our table) plays a key role in the conclusion of a study, sometimes reversing the apparent relationship (as in Simpson's Paradox).

An important observation is that there is a geometry to the way in which these errors occur. Notice that the reversal in Table~\ref{tab: simpsons paradox} would be increased by further increasing the probability of rows 1 and 4, e.g., by changing the $3,7$ counts to $300, 700$. This reweighting strengthens the dependence of $T$ on $X$. While we will not dive further into the geometry of Simpson's paradox, its existence motivates the structures we will study within networks of contexts.

\section{Omitted Label Contexts} \label{sec: omitted label contexts}
Now that we understand the potential effects of reweighting distributions on covariates, we will move our focus to the study of omitted label contexts. As discussed in the introduction, omitted label contexts involve the removal of some labels while preserving the relative probabilities of the non-removed labels. This removal can shift the apparent distribution of any variable that is associated with $Y$, including both treatment $T$ and covariates $X$.

\subsection{Causality within Omitted Label Contexts}
Consider a second hypothetical dataset that augments Table~\ref{tab: simpsons paradox} with an additional column, shown in Table~\ref{tab: augmented simpsons}. When $y^{(2)}$ is excluded, the observed dataset is equivalent to Table~\ref{tab: simpsons paradox}, which we recall has an ATE of $.3$ on the outcome of $y^{(1)}$ (relative to $y^{(0)}$). 

Although the full context has the exact same (uniform) marginal probability distribution on $X$, the inclusion of the $y^{(2)}$ column now requires that we weigh up the 1st and 4th rows (rather than the 2nd and 3rd rows). See the right side of Table~\ref{tab: augmented simpsons} for the resulting (weighted) counts. When we use these weights to compute the ATE on $y^{(1)}$, we see a reversal of the treatment effect:
\begin{equation}
    \text{ATE} = \frac{\nicefrac{1}{100} + \nicefrac{3}{10}}{2} - \frac{\nicefrac{7}{10} + \nicefrac{0}{100}}{2} = -.195.
\end{equation}

This can be understood by graphically modeling the selection bias as in \citet{bareinboim2015recovering}, shown in Figure~\ref{fig: exchangeability violation diagram}. Figure~\ref{fig: exchangeability violation diagram} (a) shows the graph describing a confounding variable $X$ causing both $T$ and $Y$. The goal of IPW and the backdoor adjustment is to reweigh the distribution to fit the DAG in Figure~\ref{fig: exchangeability violation diagram} (b), i.e. the distribution of $X$ is exchangeable in both $t^{(0)}$ and $t^{(1)}$ or equivalently $X \indep T$. Figure~\ref{fig: exchangeability violation diagram} (c) shows the effect of restricting the labels of $Y$ within a dataset context (such as with omitted label contexts), which involves conditioning on a child of $Y$. That is, let $C$ be an indicator for $Y$'s membership in the restricted set of labels and condition on $C=1$ to induce omitted label bias. $X$ and $T$ are not d-separated\footnote{see \citet{pearl2009causality} for more information on d-separation.} because conditioning on a variable that is causally downstream of both $X$ and $T$ can induce a spurious correlation between them.

\begin{figure}[t]
 \centering
  (a)\scalebox{.4}{
 \begin {tikzpicture}[-latex ,auto ,node distance =2 cm and 2 cm ,on grid , ultra thick, state/.style ={ ultra thick, circle, draw, minimum width =.85 cm}, cstate/.style ={ fill=black, text=white}]
  \node[state] (X) {$X$};
  \node[state] (T)[below left=of X] {$T$};
  \node[state] (Y)[below right=of X] {$Y$};
  \path[very thick] (X) edge (Y) (X) edge (T) (T) edge (Y);
 \end{tikzpicture}
 }
 \hfill(b)\scalebox{.4}{
 \begin {tikzpicture}[-latex ,auto ,node distance =2 cm and 2 cm ,on grid , ultra thick, state/.style ={ ultra thick, circle, draw, minimum width =.85 cm}, cstate/.style ={ fill=black, text=white}]
  \node[state] (X) {$X$};
  \node[state] (T)[below left=of X] {$T$};
  \node[state] (Y)[below right=of X] {$Y$};
  \path[very thick] (X) edge (Y) (T) edge (Y);
 \end{tikzpicture}
 }
 \hfill(c)\scalebox{.4}{
 \begin {tikzpicture}[-latex ,auto ,node distance =2 cm and 2 cm ,on grid , ultra thick, state/.style ={ ultra thick, circle, draw, minimum width =.85 cm}, cstate/.style ={ fill=black, text=white}]
  \node[state] (X) {$X$};
  \node[state] (T)[below left=of X] {$T$};
  \node[state] (Y)[below right=of X] {$Y$};
  \node[state, cstate] (C)[right=of Y] {$C$};
  \path[very thick] (X) edge (Y) (T) edge (Y) (Y) edge (C);
 \end{tikzpicture}
 }
 \caption{(a) A causal DAG depicting confounding from a common cause $X$. (b) The causal DAG that ``severs'' $X \rightarrow T$ by reweighting for exchangeability. (c) The causal DAG depicting the effect of an omitted label context $C$, which has been conditioned on. } \label{fig: exchangeability violation diagram} 
\end{figure}
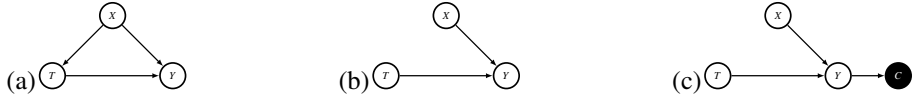

This effect darkens the outlook for causal inference in omitted label contexts, particularly because the globally correct reweighting of Table~\ref{tab: augmented simpsons} involves counts that do not appear exchangeable when we omit the $y^{(2)}$ column. We are no longer guided by the principle of exchangeability in our observed data. Finding the correct adjustment requires knowledge of the context-induced distribution  $\Pr(X, T, Y \given y \in \{y^{(i)}, y^{(i)}\})$. This distribution cannot be obtained without extending the study to all labels.

In response to this difficulty, we will study what can be learned when we have many \emph{different} omitted label contexts. These contexts do not provide access to the full joint distribution that would be needed to correctly compute a backdoor adjustment. However, the aggregation of these contexts provides partial information that can be formalized.

\section{Networks of Contexts} \label{sec:networks of contexts}
Before we discuss the structures within networks of omitted label contexts, we will introduce another paradox from social choice theory: the Condorcet Paradox \citep{nicolas1785essai}. We will see that this paradox and its structure are deeply related to networks of conclusions drawn from omitted label contexts.
\subsection{The Condorcet Paradox}
The Condorcet paradox works as follows: three voters each have preferences $y^{(0)} \rightarrow y^{(1)} \rightarrow y^{(2)}$, $y^{(1)} \rightarrow y^{(2)} \rightarrow y^{(0)}$, and $y^{(2)} \rightarrow y^{(0)} \rightarrow y^{(1)}$, with $a \rightarrow b$ indicating a preference of $a$ over $b$. These preferences are generated by rotating a starting ordering three times, meaning that $y^{(i)} \rightarrow y^{(i + 1 \mod 2)}$) in two out of the three voters. The result is an aggregate cycle of preference $y^{(0)} \rightarrow y^{(1)} \rightarrow y^{(2)} \rightarrow y^{(0)}$ with frequencies of $\nicefrac{2}{3}$ voters for each edge.

This paradox can be generalized into what we will call an ``aggregation of rankings'' (AR) --- a complete directed-graph\footnote{These graphs are always complete, but we use graph terminology as in \citet{mazaheri2021synthesizing} in order to reference properties that are dependent on cycles.} on the set of labels $\mathcal{Y}$ with weights on each $y^{(i)} \rightarrow y^{(j)}$ corresponding to the fraction of voters who prefer $y^{(i)}$ to $y^{(j)}$. AR structures are a convex combination of total orderings (i.e., graphs with edge weights of $0$ or $1$), with component weights corresponding to the fraction of voters carrying each total ordering. See Figure~\ref{fig:situational_LOG_decomp} for an illustration of this perspective for the Condorcet paradox. As a result, the space occupied by all possible AR structures is known as the ``linear ordering polytope,'' which has been the subject of extensive study \citep{fishburn1992induced, alon2002voting}.

\input{figures/situational_LOG_decomp}

The preferences of voters in the Condorcet paradox can be embedded into a table of frequencies, with each voter becoming a specific value for covariate $X$. Table~\ref{tab: condorcet paradox} demonstrates this using the counts $2 > 1 > 0$ to induce high, medium, and low preference. Notice that the order of preferences for each $x$ in the $t^{(1)}$ half of the table (first three rows) exactly correspond to the order of preferences given by the voters in the Condorcet paradox, starting with $y^{(0)} \rightarrow y^{(1)} \rightarrow y^{(2)}$ for $x^{(0)}$ and shifting the order for the other values of $X$. 

The $t^{(0)}$ half of the table complements the $t^{(1)}$ half so that the counts for $(t^{(0)}, x^{(i)}, y^{(j)})$ and $(t^{(1)}, x^{(i)}, y^{(j)})$ always sum to three. As a result, restricting our table to any two columns still yields a uniform probability distribution on $X$, i.e. $\Pr(x^{(0)} \given y \in \{y^{(i)}, y^{(j)}\}) = \Pr(x^{(1)} \given y \in \{y^{(i)}, y^{(j)}\}) = \Pr(x^{(2)} \given y \in \{y^{(i)}, y^{(j)}\})$. This is the distribution that a naive study would average over when applying a backdoor adjustment, meaning that
\begin{equation}\label{eq:interventional distributions}
\begin{aligned}
    \Pr(y^{(0)} \given \Do(t^{(1)}), y \in \{y^{(0)}, y^{(1)}\}) &= \frac{\nicefrac{2}{3} + \nicefrac{0}{2} + \nicefrac{1}{1}}{3} = \nicefrac{5}{9}\\
    \Pr(y^{(0)} \given \Do(t^{(0)}), y \in \{y^{(0)}, y^{(1)}\}) &= \frac{\nicefrac{0}{3} + \nicefrac{2}{2} + \nicefrac{1}{3}}{3} = \nicefrac{4}{9}.\\
\end{aligned}
\end{equation}
The calculations in Equation~\ref{eq:interventional distributions} conclude that the ATE on $y^{(0)}$ in the $y\in \{y^{(0)}, y^{1)}\}$ context is $+ \nicefrac{1}{9}$. These calculations are the same for the ATE on $y^{(1)}$ for $y\in \{y^{(1)}, y^{(2)}\}$ and the ATE on $y^{(2)}$ for $y\in \{y^{(2)}, y^{(0)}\}$ due to the cyclic shifting of columns. Hence, the studies separately conclude that the treatment increases the relative frequency of all three labels, which is clearly impossible.
The embedding of the Condorcet paradox into causal conclusions implies a correspondence between aggregations of rankings and backdoor adjustments (or any other case-based weighting).

\paragraph{Cardinal vs Ordinal Systems}
The Condorcet paradox is primarily driven by the loss of information in an ordinal system. That is, an ordering of $A \rightarrow B \rightarrow C$ cannot distinguish between the magnitude of the preferences $A \rightarrow B$ and $A \rightarrow C$. If voters were instead allowed to allocate multiple votes among the candidates (e.g. one voter gives 2 votes for $A$, 1 for $B$, and 0 for $C$), it would be clear that $A \rightarrow C$ is a stronger opinion than $A \rightarrow B$ (that is, $2-1 < 2-0$). Indeed, such ``cardinal'' voting systems do not give rise to paradoxes of transitivity \citep{conklin1923comparison}.

Consider mapping the paradox we have described here to a weighted DAG structure. The edge-weight on $y^{(i)} \rightarrow y^{(j)}$ corresponds to a do-intervention that has been computed on a context that omits all labels other than $y^{(i)}, y^{(j)}$. Since these edge-weights are constructed using ordinal values (i.e., the probability distribution on the $Y$), we would not expect the Condorcet paradox to emerge. Nonetheless, we will show that these ``aggregations of soft rankings'' occupy the same linear ordering polytope as ``aggregations of rankings.''

\subsection{Aggregations of Rankings and Soft Rankings}
Table~\ref{tab: condorcet paradox} used counts of $2, 1, 0$ to induce preference between labels in each row. Because this system is cardinal, the ``preference'' (the relative frequency of each label) can take on any value in $[0,1]$. For this reason, we refer to the induced preferences in each row of our tables as a ``soft ranking.'' We will now formally define both aggregations of rankings (ARs) and aggregations of soft rankings (ASRs).

\begin{definition}[Ranking]
    A ranking of $\mathcal{Y}$ is a function $A : \mathcal{Y} \by \mathcal{Y} \rightarrow \{0, 1\}$ generated by a total ordering. We use $A(y^{(i)}, y^{(j)}) = 1$ to denote preference $y^{(i)} \rightarrow y^{(j)}$ and $A(y^{(i)}, y^{(j)}) = 0$ for $y^{(i)} \leftarrow y^{(j)}$.
\end{definition}
\begin{definition}[Aggregation of Rankings (AR)]
An aggregation of rankings is specified by a set of rankings $\vec{A}$ and a corresponding weight function $\vec{\alpha} \in \simplex^{\abs{\vec{A}}}$ (indexed by $A \in \vec{A}$).
\end{definition}
\begin{definition}[Aggregate Preference]
An aggregation preference in an AR between $y^{(i)}, y^{(j)} \in \mathcal{Y}$ is defined to be
\begin{equation*}
    R_{\vec{A}, \vec{\alpha}}(y^{(i)}, y^{(j)}) \coloneq \sum_{A \in \vec{A}} \vec{\alpha}_A A(y^{(i)}, y^{(j)}).
\end{equation*}
\end{definition}
Corresponding to rankings, ARs, and aggregate preferences $R_{\vec{A}, \vec{\alpha}}$, we will have soft rankings, ASRs, and aggregate probabilities $F_{\vec{B}, \beta}$.

\begin{definition}[Soft Rankings]
    A soft ranking on $\mathcal{Y}$ is a function $B : \mathcal{Y} \by \mathcal{Y} \rightarrow [0, 1]$ generated by a categorical probability distribution on $\mathcal{Y}$, $\vec{p} \in \simplex^{\abs{\mathcal{Y}}}$.
    \begin{equation*}
        B(y^{(i)}, y^{(j)}) \coloneq \frac{p_i}{p_i + p_j}.
    \end{equation*}
\end{definition}
\begin{definition}[Aggregation of Soft Rankings (ASR)]
An aggregation of soft rankings is specified by a set of soft rankings $\vec{B}$ and a corresponding weight function $\beta \in \simplex^{\abs{\vec{B}}}$ (indexed by $B \in \vec{B}$).
\end{definition}

\begin{definition}[Aggregate Probability]
An aggregate probability in an ASR between $y^{(i)}, y^{(j)} \in \mathcal{Y}$ is
\begin{equation*}
    F_{\vec{B}, \beta}(y^{(i)}, y^{(j)}) \coloneq \sum_{B \in \vec{B}} \beta_B B(y^{(i)}, y^{(j)}).
\end{equation*}
\end{definition}

\begin{observation}\label{obs: connection}
Suppose the probability distribution for a covariate adjustment on $\vec{X}$ (e.g. $X$ in our previous examples), $\Pr(\vec{X} \given Y \in \{y^{(i)}, y^{(j)}\})$, is the same for all pairs of labels $\{y^{(i)}, y^{(j)}\}$. The treatment effects computed for each label pair then correspond to the difference between the aggregate probabilities in two ASRs with a $B \in \vec{B}$ for each assignment of $\vec{X} = \vec{x}$ and $\beta_B = \Pr(\vec{x} \given Y \in \{y^{(i)}, y^{(j)}\})$.
\end{observation}

Observation~\ref{obs: connection} illustrates the connection between ASRs and the paradox of nontransitivity developed in the previous section --- namely, that the backdoor adjustment is a weighted sum of conditional probabilities that can each be interpreted as a soft ranking. The remainder of this paper will be dedicated to showing that ARs and ASRs on the same cardinality $\abs{\mathcal{Y}}=n$ can hold the exact same vectors of weights. This gives a correspondence between the notion of ``consistency'' when combining causal quantities in omitted label contexts and aggregations of rankings in the already well-studied linear ordering polytope.

To make our statement precise, we will denote $\mathcal{A}$ as the set of $\{0,1\}^{n(n-1)}$ vectors associated with the output values of some $A$ in a total ordering and $\co(\mathcal{A})$ as its convex hull. Note that $\co(\mathcal{A})$ is the space of edge-weights for ARs, i.e., vectors with entries corresponding to $R(y^{(i)}, y^{(j)})$. Similarly, denote $\mathcal{B}$ as the set of $[0, 1]^{n(n-1)}$ vectors generated by some categorical distribution and observe that its convex hull $\co(\mathcal{B})$ is the space of possible edge-weights for ASRs.

\begin{theorem}\label{thm: equivalence}
    $\co(\mathcal{A})$ and $\cco(\mathcal{B})$ are the same.
\end{theorem} 

It is not difficult to see how soft rankings can be made ``harder'' by increasing the relative difference in probabilities. That is, replacing $2, 1, 0$ in Table~\ref{tab: condorcet paradox} with $100, 1, 0$ more closely simulates an absolute preference. Showing that any set of aggregate probabilities from an ASR can be realized with aggregate preferences from an AR is less obvious. We will prove this direction by using the probability table in an ASR to directly construct a corresponding AR.

\subsection{Probabilities can Emulate Preferences}
We will begin with the simpler direction, given by Lemma~\ref{lem:AProb to APref}.
\begin{lemma}
   \label{lem:AProb to APref}
   $\co(\mathcal{A}) \subset \cco (\mathcal{B})$.
\end{lemma}
To prove Lemma~\ref{lem:AProb to APref}, we will first show that for every $A \in \mathcal{A}$, there exists a $B \in \mathcal{B}$ that is arbitrarily close to it. We will then make use of the following more general lemma, which we prove in Appendix~\ref{apx: proof of convex hull thm}.
\begin{lemma}
\label{thm:convex_hulls}
Consider a set of vectors $\mathcal{V} = \{\vec{v}_1, \ldots, \vec{v}_t \}$ and with $\vec{v}_i \in \R^m$ for all $i$. If we have $\Vt$ such that for every $\varepsilon >0$ and $\vec{v} \in V$, there exists $\vec{\vt} \in \Vt$ such that $\twonorm{\vec{\vt} - \vec{v}} < \varepsilon$, then $\co(\mathcal{V}) \subseteq \cco(\Vt)$.
\end{lemma}

We now give the proof for Lemma~\ref{lem:AProb to APref}. 
\begin{proof}
    For a given $A$ from a total ordering, we will show how to find a probability vector $\vec{p}$ that generates a $B$ with values that are arbitrarily close to the $0, 1$ values of $A$. As already alluded to, this will involve blowing up the ratios of the probabilities in $\vec{p}$.
    
    Let $y^{(0)} \rightarrow \ldots \rightarrow y^{(n-1)}$ be the ordering specified by $A$. Let the $i$th element of $\vec{p}$ be $\varepsilon^{i}/z$, where $z = \sum_{j=1}^n \varepsilon^j$ is a normalization factor so that $\vec{p}$ remains in the simplex. Notice that this assignment gives us
    \begin{equation*}
        B(y^{(i)}, y^{(j)}) = \begin{cases}
            \frac{1}{1 + \varepsilon^{j-i}} & \text{ if } i < j\\
            \frac{\varepsilon^{i - j}}{1 + \varepsilon^{i-j}} & \text{ if } i > j
        \end{cases}.
    \end{equation*}
    Setting $\varepsilon>0$ arbitrarily close to $0$ achieves
    \begin{equation*}
        B(y^{(i)}, y^{(j)}) \longrightarrow \begin{cases}
            1 & \text{ if } i < j\\
            0 & \text{ if } i > j\\
        \end{cases},
    \end{equation*}
    which means that $B(y^{(i)}, y^{(j)}) \rightarrow A(y^{(i)}, y^{(j)})$.
    Finally, we apply Lemma~\ref{thm:convex_hulls}.
\end{proof}

\subsection{Preferences can Emulate Probabilities}

We will continue with the more difficult direction, given by Lemma~\ref{lem:APref to Aprob}.
\begin{lemma}
   \label{lem:APref to Aprob}
   $\co(\mathcal{B}) \subset \co(\mathcal{A})$.
\end{lemma}

We will prove this direction by showing that every possible instance of $B$ is in $\co(\mathcal{A})$. Convexity of $\co(\mathcal{A})$ will then complete the proof.

Let $\vec{A}^{(i)} \subset \vec{A}$ denote the set of rankings for which $y^{(i)}$ is the ``first choice.'' Equivalently, $\vec{A}^{(i)}$ is defined such that we have $A(y^{(i)}, y^{(j)}) = 1$ for all $j \neq i$ and $A \in \vec{A}^{(i)}$. We extend this notation to multiple indices, with $\vec{A}^{(ij)}$ encoding $y^{(i)}$ as first choice and $y^{(j)}$ as second choice (i.e., $y^{(j)}$ is preferred over all labels other than $y^{(i)}$). If only one ranking satisfies the restriction, then we will use an unbolded $A^{(ij)}$ to emphasize that we are no longer dealing with a set.

%%%

%inductive proof
\input{figures/EG_to_LOP_decomp}

\begin{proof}
    We will induct on the number of labels $n$. The inductive hypothesis is that any $B \in \mathcal{B}$ generated by a categorical distribution $\vec{p} \in \simplex^n$ over $n$ labels can be expressed as an AR $\vec{A}$, $\vec{\alpha}$. This can easily be shown for the base case of $n=2$ by assigning $\vec{\alpha}_{A^{(01)}} = p_0$ and $\vec{\alpha}_{A^{(10)}} = p_1$.

    Now, assuming the inductive hypothesis to be correct for all $B$ on $n$ labels, we will show how to construct an AR for a $B$ on $n+1$ labels. First, expand $R_{\vec{A}, \vec{\alpha}}$, which we have not yet specified, into aggregate rankings on $\vec{A}^{(0)}, \ldots, \vec{A}^{(n)}$,
    \begin{equation} \label{eq: decomp}
        R_{\vec{A}, \vec{\alpha}} = \sum_{k=0}^{n} p_k R_{\vec{A}^{(k)}, \vec{\alpha}^{(k)}}.
    \end{equation}

     Now, choose some label $y^{(k)}$ and construct a new $\vec{p}^{[k]} \in \simplex^n$ that will represent the ``conditional'' probability table for all rankings with $y^{(k)}$ as their first choice. This is done by setting all $p^{[k]}_i = p_i/(1-p_k)$, for $i \neq k$.
     This new $\vec{p}^{[k]}$ implies a $B^{[k]}(y^{(i)}, y^{(j)})$ that matches $B(y^{(i)}, y^{(j)})$ for all $i,j \neq k$. $B^{[k]}$ also satisfies the requirements for the inductive hypothesis when leaving out $y^{(k)}$ from the list of labels, thereby reducing the length of the list of labels. We can therefore assume there is a set of rankings $\vec{A}^{[k]}$ and corresponding $\vec{\alpha}^{[k]}$ that forms an AR for which $R_{\vec{A}^{[k]}, \vec{\alpha}^{[k]}} = B^{[k]}$. 
     
     Each $A^{[k]} \in \vec{A}^{[k]}$ can now be augmented with a first-choice preference of $y^{(k)}$ to generate the set $\vec{A}^{(k)}$ with corresponding $\vec{\alpha}^{(k)} = \vec{\alpha}^{[k]}$. Using this assignment, we have that
     \begin{equation} \label{eq: cases in decomp}
         R_{\vec{A}^{(k)}, \vec{\alpha}^{(k)}}(y^{(i)}, y^{(j)}) = 
         \begin{cases}
             B(y^{(i)}, y^{(j)})  & \text{ if } i,j \neq k\\
             1 & \text{ if } i = k\\
             0 & \text{ if } j = k
         \end{cases}
     \end{equation}

     Applying Equation~\ref{eq: cases in decomp} to Equation~\ref{eq: decomp} verifies that
    \begin{equation*}
    \begin{aligned}
        R_{\vec{A}, \vec{\alpha}}(y^{(i)}, y^{(j)}) &= p_i + \sum_{k \neq i,j} p_k  B(y^{(i)}, y^{(j)}) = (p_i + p_j)B(y^{(i)}, y^{(j)}) + (1 - p_i - p_j)B(y^{(i)}, y^{(j)})\\
        &= B(y^{(i)}, y^{(j)}).
    \end{aligned}
    \end{equation*}
     Hence, our construction of $\vec{p}_i^{[k]}$ is valid. This completes the inductive proof. As stated earlier, convexity of $\co(\mathcal{A})$ gives the desired result.
\end{proof}
Figure~\ref{fig:eg_to_lop_decomp} illustrates an example inductive step in the proof for Lemma~\ref{lem:APref to Aprob}.

\section{Discussion}\label{sec:discussion}

This paper studies dataset contexts and how they affect the conclusions we take from them. Specifically, we have explored how the traditional principles of causal inference break down in omitted label contexts through the manifestation of two well-known paradoxes. These paradoxes warn practitioners about causal inference within a label-biased setting. In addition, the structures of these paradoxes (as described by Theorem~\ref{thm: equivalence}) provide insight into their nature and define a notion of consistency between partial studies.

The study of omitted label contexts may have applications in handling other forms of label bias. In the study of rare diseases, it may be unfeasible to avoid scaling up the proportion of affected individuals. This is a form of label bias and comes with all of the caveats discussed in the preceding sections. An intriguing alternative is to instead combine the label of interest with another similarly rare disease. In a dataset of two rare diseases, one can maintain the relative probabilities of the labels while keeping their portion of the dataset nontrivial. In order to extend the results of this type of study to the broader population, one can imagine computing multiple ratios of increasingly more common labels in order to compute the context-induced distributions needed to perform adjustments. Such a combination requires a deep understanding of the structure in networks of contexts.

While the linear ordering polytope has a few counterintuitive properties, such as the breakdown of preference-transitivity, it also has limitations. For example, the Condorcet paradox actually represents the ``maximum amount of nontransitivity'' that is possible in a cycle of three choices. It is not possible for a cycle of $80 \%$ preferences to exist, nor is a cycle of $100 \%, 50\%, 70\%$ possible. In fact, for any cycle of length $\ell$, the aggregate preferences along that cycle must sum to between $1$ and $\ell-1$, a property discussed in \citet{fishburn1992induced} as the ``triangle inequality'' and \citet{mazaheri2021synthesizing} as the ``curl condition.''

These limitations can be harnessed to provide bounds on unmeasured aggregate preferences, or ``missing edge weights'' as discussed in \citet{mazaheri2021synthesizing}. That is, if we are missing a study on two labels, we can provide bounds on that study's likely outcome using the outcomes of related studies. Furthermore, we can use these properties to detect inconsistencies within sets of studies.

As larger machine learning models become more costly to train and comprehensive data drifts towards private datasets, many practitioners are choosing to re-purpose pre-trained ``checkpoints'' to new tasks. Of course, these checkpoints all come from data-contexts. Hence, we see decision fusion from different contexts as an essential upcoming field to launch data science into a new age.

% Acknowledgments---Will not appear in anonymized version
\acks{The authors thank the Paradise Lab at Caltech. Bijan Mazaheri was supported by a National Science Foundation Graduate Research Fellowship and an Amazon AI4Science fellowship.}

\bibliography{refs}

\appendix

\section{Proof of Lemma~\ref{thm:convex_hulls}}
\label{apx: proof of convex hull thm}

Convex hulls of finite sets in $\R^\ell$ are \emph{convex} polytopes, which can be expressed as an intersection of $h$ halfspaces indexed by $f$ with $\{\vec{x} : \vec{a^{(f)}}^\top \vec{x} x< b^{(f)}\}$ \citep{grunbaum1967convex}. Vectors $\vec{a^{(f)}}^\top$ can be combined as row-vectors of a matrix, $A$, so that any convex polytope can be expressed as
\begin{equation}
\label{eq:polytopes}
 \{x : A\vec{x} \prec \vec{b} \} = \left\{\vec{x} : \begin{pmatrix} (\vec{a^{(1)}})^\top\\ \vdots \\ (\vec{a^{(h)}})^\top\end{pmatrix}\vec{x} \prec \begin{pmatrix}b^{(1)}\\ \vdots \\ b^{(h)} \end{pmatrix} \right\}.
\end{equation}
For convenience, the vectors $\vec{a^{(f)}}, \vec{\at^{(f)}}$ are assumed to be unit vectors throughout.

The idea behind the proof will be to analyze the movement of the boundaries of the polytope defined by $\mathcal{V} = \{\vec{v}_1, \ldots, \vec{v}_m \}$ and corresponding polytope defined by the ``perturbed points'' $\Vt = \{\vt_1, \ldots \vt_m \}$. The key is to show that a point that is far enough from the boundary of $\co(\mathcal{V}$ will also be within $\co(\Vt)$, given by Lemma~\ref{lem:convex_analysis}. This required distance from the boundary will be relative to the amount by which the perturbed points have moved. As we make the perturbation arbitrarily small (i.e. $\varepsilon \rightarrow 0$, all points in the interior of the polytope will be included.
\begin{lemma}
\label{lem:convex_analysis}
Let
\begin{align*}
    \co(\mathcal{V}) &= \{\vec{x} : A\vec{x} \prec \vec{b}\}\\
    \co(\Vt) &= \{\vec{x} : \At \vec{x} \prec \vec{\bt}\}
\end{align*}
as given by Equation~\ref{eq:polytopes}. If $A \vec{x} \prec \vec{b} - \varepsilon \onevec$ and $\twonorm{\vec{v_i} - \vec{\vt_i}} < \varepsilon\; \forall i$, then $\At \vec{x} \prec \vec{\bt}$.
\end{lemma}

To prove Lemma~\ref{lem:convex_analysis}, we will need to show that the boundaries of the polytopes do not move too much. We will do this using Lemma~\ref{lem:perturbed_boundaries}, which bounds how far $\bo(\co(V))$ can be from $\bo(\co(\Vt))$ along a single ``face.'' 

\begin{definition}
\label{definition:arbitrary_face}
Choose $f \in [h]$. Define:
\begin{align*}
    W^{(f)} &= \{\vec{w} : (\vec{a^{(f)}})^\top w = b^{(f)}, \vec{w} \in V\}\\
    \Wt^{(f)} &= \{\vec{\vt_i} : \vec{v_i} \in W^{(f)}\}
\end{align*}
We restrict the size of $\abs{W^{(f)}} = \ell$, which is the number of points needed to define a halfspace in $\R^\ell$. This can be done by allowing for multiple identical $\vec{a_f}, b_f$ combinations corresponding to all size $\ell$ subsets of the $v_i$ along the boundary. 

Note that $\co(W^{(f)})$ describes a ``face'' of the polytope $\co(\mathcal{V})$ indexed by $f$ which is perpendicular to $\vec{a^{(f)}}$. $\co(\Wt^{(f)})$ describes the perturbed face.
\end{definition}

\begin{lemma}
\label{lem:perturbed_boundaries}
Choose $f, g \in [h]$ arbitrarily and let $W^{(f)} = \{\vec{w^{(f)}_1}, \ldots, \vec{w^{(f)}_\ell}\}$ and $\Wt^{(f)} = \{\vec{\wt^{(f)}_1}, \ldots, \vec{\wt^{(f)}_\ell}\}$.  
For every $\vec{m^{(f)}} \in \cco(W^{(f)})$, we have $(\vec{\at^{(g)}})^\top \vec{m^{(f)}} < \bt^{(g)} + \varepsilon$.
\end{lemma}
\begin{proof}
Because $m \in \cco(W^{(f)})$, there is some $\lambda \in \triangle_\ell$ with
\begin{equation}
    \vec{m^{(f)}} = \sum_{i=1}^{\ell} \lambda_i \vec{w^{(f)}_i} \in \cco(W^{(f)})
\end{equation}
Consider also
\begin{equation}
    \vec{\mt^{(f)}} = \sum_{i=1}^{\ell} \lambda_i \vec{\wt^{(f)}_i} \in \cco(\Wt^{(f)})
\end{equation}
Note that the norm of the difference between these two vectors is bounded:
\begin{equation}
\begin{aligned}
    \twonorm{\vec{m^{(f)}} - \vec{\mt^{(f)}}} &= \twonorm{\sum_{i=1}^{\ell} \lambda_i (\vec{w^{(f)}_i} - \vec{\wt^{(f)}_i})} \\
    &\leq \sum_{i=1}^{\ell} \lambda_i \underbrace{\twonorm{\vec{w^{(f)}_i} - \vec{\wt^{(f)}_i}}}_{< \varepsilon} < \varepsilon
\end{aligned}
\end{equation}
Also note that because $\vec{\mt^{(f)}}\in \cco(\Wt^{(f)}) \subseteq \cco(\Vt)$, we have that $(\vec{\at^{(g)}})^\top \vec{\mt^{(f)}} \leq \bt^{(g)}$. Now, a simple application of Cauchy-Schwartz gives:
\begin{equation}
    \begin{aligned}
        (\vec{\at^{(g)}})^\top \vec{m^{(f)}} &= (\vec{\at^{(g)}})^\top(\vec{\mt^{(f)}} + (\vec{m^{(f)}} - \vec{\mt^{(f)}}))\\
        &= \underbrace{(\vec{\at^{(g)}})^\top \vec{\mt^{(f)}}}_{\leq \bt^{(g)}} + (\vec{\at^{(g)}})^\top (\vec{m^{(f)}} - \vec{\mt^{(f)}})\\
        &\leq \bt^{(g)} + \twonorm{\vec{\at^{(g)}}}\twonorm{\vec{m^{(f)}} - \vec{\mt^{(f)}}}\\
        &< \bt^{(g)} + \varepsilon
    \end{aligned}
\end{equation}
\end{proof}

With this, we are now ready to prove Lemma~\ref{lem:convex_analysis}.
\begin{proof}
Choose an arbitrary face $g \in [h]$. Recall we have $\vec{x} \in \co(V)$ with $(\vec{a^{(g)}})^\top \vec{x} < b - \varepsilon$ and we wish to show $(\vec{\at^{(g)}})^\top \vec{x} < \bt^{(g)}$.

Let $\vec{m^{(f)}_x}$ be the result of extending $\vec{\at^{(g)}}$ from $\vec{x}$ to $\bo(V)$. This must hit some face with $(\vec{a^{(f)}})^\top \vec{m^{(f)}_x} = b^{(f)}$, so $\vec{m^{(f)}_x} \in \co(W^{(f)})$. That is, find $\beta$ such that
\begin{equation}
        \vec{m^{(f)}_x} = \beta \vec{\at^{(g)}} + \vec{x} \in \co(W^{(f)})
\end{equation}

First, lets bound $\beta$. Notice that because $\vec{m^{(f)}_x} \in \co(W^{(f)})$, we have
\begin{equation}
\begin{aligned}
    (\vec{a^{(f)}})^\top\vec{m^{(f)}_x} &= (\vec{a^{(f)}})^\top \left( \sum_{i=1}^{\ell} \lambda_i \vec{w^{(f)}_i}\right) \\
    &= \sum_{i=1}^{\ell} \lambda_i(\vec{a^{(f)}})^\top \vec{w^{(f)}_i} = b^{(f)}
\end{aligned}
\end{equation}
So, we have
\begin{equation}
    b^{(f)} = (\vec{a^{(f)}})^\top\vec{m^{(f)}_x} = \beta \underbrace{(\vec{a^{(f)}})^\top \vec{\at^{(g)}}}_{\leq 1} + \underbrace{(\vec{a^{(f)}})^\top \vec{x}}_{< b^{(f)} - \varepsilon} \Rightarrow \varepsilon < \beta
\end{equation}
Now, apply Lemma~\ref{lem:perturbed_boundaries}
\begin{equation}
    \begin{aligned}
        (\vec{\at^{(g)}})^\top \vec{m^{(f)}_x} &< \bt^{(g)} + \varepsilon\\
        (\vec{\at^{(g)}})^\top \vec{x} + (\vec{\at^{(g)}})^\top \vec{\at^{(g)}} \beta &< \bt^{(g)} + \varepsilon\\
        (\vec{\at^{(g)}})^\top \vec{x} &< \bt^{(g)}.
    \end{aligned}
\end{equation}
Face $g \in [h]$ was chosen arbitrarily, so this holds for all half-spaces in the convex polytope. Hence, we have $A \vec{x} \prec \vec{b}$.
\end{proof}

\section{Counts that Follow the Correct DAG}
\label{apx: details of causality paradox}

The paradox presented in Table~\ref{tab: condorcet paradox} used counts yielding a distribution that does not precisely follow the given DAG. We will now show how to construct a similar set of counts with the statistics needed to imply the causal structure.

Consider Table~\ref{tab:nontransitivity for causality more precise}. The structure is copied from Table~\ref{tab: condorcet paradox}.
If $\alpha_1 = \alpha_2$, $\beta_1 = \beta_2$, and $\gamma_1 = \gamma_2$, then we notice that the relative probabilities of $\Pr(x)$ are given by the $\alpha$s, $\beta$s, and $\gamma$s. If these coefficients are all equal, then we have every row considered with equal weight, as in the main paper.

While setting all of the Greek coefficients to $1$ provides a nice intuition for how the paradox emerges, it does not give a distribution that obeys the requirements of the given DAG. In order for our distribution to (1) factorize according to the DAG and (2) be faithful to the DAG, we must have the following properties:
\begin{enumerate}
    \item $T \not \indep Y$
    \item $X \not \indep Y$
    \item $T \not \indep X$
    \item $T \not \indep Y \given X$
    \item $X \not \indep Y \given T$
    \item $T \not \indep X \given Y$
\end{enumerate}

With all of the Greek coefficients set to $1$, we notice that conditions 5,6 are met. The domain expertise setting effectively conditions on $Y$ by restricting its values. When restricted to two columns, we also meet condition 4.

The remaining conditions (as well as condition 4 in the broader case) can be met by varying the Greek coefficients. A Jupyter notebook available \href{https://github.com/honeybijan/causal_paradoxes_sup}{here} allows one to explore different settings to the Greek coefficients to achieve this paradox. One example is $\alpha_1 = \beta_1 = \gamma_1 = 1$ and $\alpha_2 = 1.1$, $\beta_2 = 1.2$, $\gamma_2 = 1.3$. The code returns whether the 6 independence conditions hold. The three ATEs for each omitted label context are then printed. All code was run on a Macbook Air with an M1 processor.

\begin{table}
\centering
\begin{tabular}{|c|c|c|c|c|}
        \hline
        $T$ & $X$ & $y^{(0)}$ & $y^{(1)}$ & $y^{(2)}$\\
        \hline
        $t^{(1)}$ & $x^{(1)}$ & $2\alpha_1 $ & $ \alpha_1$ & $ 0 $\\
        \hline
        $t^{(1)}$ & $x^{(2)}$ & $0$ & $2\beta_1$ & $  \beta_1$ \\
        \hline
        $t^{(1)}$ & $x^{(3)}$ & $ \gamma_1$ & $0 $ & $2 \gamma_1$ \\
        \hline
        $t^{(0)}$ & $x^{(1)}$ & $0$ & $\alpha_2 $ & $2 \alpha_2$\\
        \hline
        $t^{(0)}$ & $x^{(2)}$ & $ 2 \beta_2$ & $0$ & $\beta_2$ \\
        \hline
        $t^{(0)}$ & $x^{(3)}$ & $\gamma_2$ & $2 \gamma_2$ & $0$ \\
        \hline
\end{tabular}
\caption{A more general specification of counts for our paradox.}
    \label{tab:nontransitivity for causality more precise}
\end{table}

\end{document}

%% file: packages.tex
\usepackage{tikz}
\usepackage{mathtools}
\usepackage{nicefrac}
\usepackage{amsmath}
\usetikzlibrary{matrix, positioning}
\usepackage{xcolor,stackengine}
\usepackage{array}
\usepackage{color, colortbl}
\usepackage{listings}
\usepackage{mathtools}
\usepackage{verbatim}
\usepackage{subcaption}
\usepackage{array}
\usepackage{graphicx}
\usepackage[utf8]{inputenc} % allow utf-8 input
\usepackage[T1]{fontenc}    % use 8-bit T1 fonts
\usepackage{hyperref}       % hyperlinks
\usepackage{booktabs}       % professional-quality tables
\usepackage{amsfonts}  % blackboard math symbols

\usepackage{microtype}      % microtypography
\usepackage{xcolor}         % colors
\usepackage{ bbold }
\usepackage{bm}
\usepackage{enumerate}
\usepackage[scr=boondox,  % heavily sloped
            cal=esstix]{mathalpha}

\usepackage[utf8]{inputenc}
\usepackage[T1]{fontenc}
\usepackage{newtxtext,newtxmath}
\usepackage{hyperref}

%% file: shortcuts.tex
\newcommand{\by}{\times}

\newcommand{\indep}{\perp \!\!\! \perp}

\newcommand{\R}{\mathbb{R}}

\def\abs#1{\left| #1 \right|}

\newcommand{\Vt}{\tilde{\mathcal{V}}}
\newcommand{\vt}{\tilde{\vec{v}}}
\newcommand{\bt}{\tilde{b}}
\newcommand{\at}{\tilde{a}}
\newcommand{\At}{\tilde{A}}
\newcommand{\Wt}{\tilde{W}}
\newcommand{\wt}{\tilde{w}}
\newcommand{\mt}{\tilde{m}}% i
\newcommand{\onevec}{\vec{1^{\ell}}}
\newcommand{\co}{\text{Co}}
\newcommand{\cco}{\overline{\text{Co}}}
\newcommand{\bo}{\text{Bo}}

\renewcommand{\vec}{\bm}

\newcommand{\given}{\mid}

\newcommand{\PreserveBackslash}[1]{\let\temp=\\#1\let\\=\temp}
\newcolumntype{C}[1]{>{\PreserveBackslash\centering}p{#1}}
\newcolumntype{X}[2]{>{\columncolor{#2}\PreserveBackslash\centering}p{#1}}
\renewcommand\vec{\mathbf}

\newcommand{\YY}[1]{y^{(#1)}}

\newcommand{\Classes}{\vec{Y}}

\newcommand{\norm}[1]{\left\lVert#1\right\rVert}
\newcommand{\twonorm}[1]{\norm{#1}_{2}}
\newcommand{\Do}{\textrm{do}}
\renewcommand{\do}{\textrm{do}}

\newtheorem{observation}{Observation}

\usetikzlibrary{decorations,decorations.pathmorphing, decorations.pathreplacing, decorations.shapes, arrows, arrows.meta, calc}
\pgfdeclarelayer{bg}    % declare background layer
\pgfsetlayers{bg,main}

\pgfdeclaremetadecoration{middlezigzag}{straight}{
    \state{straight}[switch if less than=\pgfmetadecorationsegmentlength to final,
                  width=\pgfmetadecoratedpathlength/2 - \pgfmetadecorationsegmentlength/2 ,
                  next state=zigzag] {
        \decoration{curveto}
    }
    \state{zigzag}[width=\pgfmetadecorationsegmentlength,
                   next state=final] {
        \decoration{zigzag}
    }
    \state{final}{
        \decoration{curveto}
        \beforedecoration{\pgfpathmoveto{\pgfpointmetadecoratedpathfirst}}
    }
}

\tikzset{
    middle zigzag/.style={
        decorate,
        decoration={
            middlezigzag,
            meta-segment length=.36cm,
            segment length=0.12cm,
            amplitude = .05cm,
        }
    }
}

\makeatletter
\newcommand{\bigcomp}{%
  \DOTSB
  \mathop{\vphantom{\sum}\mathpalette\bigcomp@\relax}%
  \slimits@
}
\newcommand{\bigcomp@}[2]{%
  \begingroup\m@th
  \sbox\z@{$#1\sum$}%
  \setlength{\unitlength}{0.9\dimexpr\ht\z@+\dp\z@}%
  \vcenter{\hbox{%
    \begin{picture}(1,1)
    \bigcomp@linethickness{#1}
    \put(0.5,0.5){\circle{1}}
    \end{picture}%
  }}%
  \endgroup
}
\newcommand{\bigcomp@linethickness}[1]{%
  \linethickness{%
      \ifx#1\displaystyle 2\fontdimen8\textfont\else
      \ifx#1\textstyle 1.65\fontdimen8\textfont\else
      \ifx#1\scriptstyle 1.65\fontdimen8\scriptfont\else
      1.65\fontdimen8\scriptscriptfont\fi\fi\fi 3
  }%
}

\def\Pr{\operatorname{Pr}}

%% file: figures/situational_LOG_decomp.tex
\newcommand{\ranker}[7]{\scalebox{.8}{\begin {tikzpicture}[-latex, auto,node distance =1.4cm and 1.4cm ,on grid ,
semithick, baseline, anchor=base]
\node[circle, draw, fill=#4, minimum width =.4 cm, inner sep=0.05pt, label= below:$\YY{#1}$] (a){};
\node[circle,draw, fill=#5,minimum width =.4 cm, inner sep=0.05pt,label= below:$\YY{#2}$] (b) [right = of a]{};
\node[circle,draw, fill=#6, minimum width =.4 cm, inner sep=0.05pt,label= below:$\YY{#3}$] (c) [right = of b]{};
\node (lab) [above = 1 cm of b] {#7};
\path (a) edge [bend left = 0]  (b);
\path (b) edge [bend left = 0]  (c);
\path (a) edge [bend left = 25] (c);
\end{tikzpicture}}}

\newcommand{\cycle}[6]{\scalebox{.7}{\begin {tikzpicture}[-latex, auto,node distance =.75cm and 1 cm ,on grid ,
semithick, baseline, anchor=base, state/.style ={ circle, draw, minimum width =.4 cm, fill}]
\node[state, fill = #4] (A) at (90:1) {};
\node[state, fill = #5] (B) at (330:1) {};
\node[state, fill = #6] (C) at (210:1) {};
\node[] (Alab) at (90:1.4) {$\YY{1}$};
\node[] (Blab) at (330:1.7) {$\YY{2}$};
\node[] (Clab) at (210:1.7) {$\YY{3}$};
\node[] (CAout) at (150:.8) {$#3$};
\node[] (BCout) at (270:.9) {$#2$};
\node[] (ABout) at (30:.8) {$#1$};
\path (A) edge (B);
\path (B) edge (C);
\path (C) edge (A);
%\path (u) edge [bend right = 0] node[above] {$0.7$} (z);
\end{tikzpicture}}}

\begin{figure*}
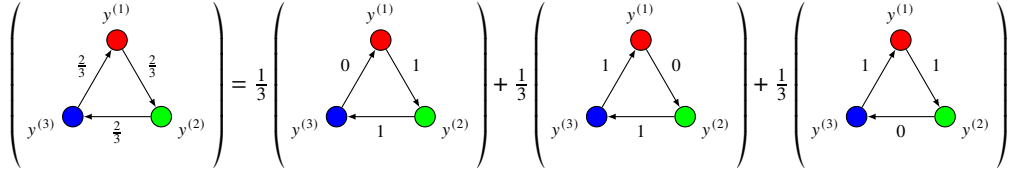

    \centering
\begin{equation*}
\resizebox{.95 \textwidth}{!}{$
    \left(\cycle{\frac{2}{3}}{\frac{2}{3}}{\frac{2}{3}}{red}{green}{blue}\right) = \frac{1}{3}\left(\cycle{1}{1}{0}{red}{green}{blue}\right) + \frac{1}{3}\left(\cycle{0}{1}{1}{red}{green}{blue}\right) + \frac{1}{3}\left(\cycle{1}{0}{1}{red}{green}{blue}\right)$}
\end{equation*}
    \caption{The Condorcet paradox as an aggregation of rankings.}
    \label{fig:situational_LOG_decomp}
\end{figure*}

%% file: figures/EG_to_LOP_decomp.tex
\begin{figure*}
\centering
\begin{align*}
\left(\scalebox{.55}{\begin {tikzpicture}[-latex, auto,node distance =1cm and 1.33 cm ,on grid,
semithick, baseline, anchor=base, state/.style ={ circle, draw, minimum width =.8 cm}]
\node (center) {};
\node[state] (c) [below left = of center]{$.2$};
\node (clab) [below left = .5cm and .5cm of c]{$\YY{2}$};
\node[state] (a) [above = of center]{$.5$};
\node (alab) [above = .7cm of a]{$\YY{0}$};
\node[state] (b) [below right = of center]{$.3$};
\node (blab) [below right = .5cm and .5cm of b]{$\YY{1}$};
\path (a) edge node[above right] {$\frac{5}{8}$} (b);
\path (b) edge node[below] {$\frac{3}{5}$} (c);
\path (c) edge node[above left] {$\frac{2}{7}$} (a);
\end{tikzpicture}}\right) &= .5 \left(\scalebox{.55}{\begin {tikzpicture}[-latex, auto,node distance =1cm and 1.33 cm ,on grid,
semithick, baseline, anchor=base, state/.style ={ circle, draw, minimum width =.8 cm}]
\node (center) {};
\filldraw[color = red, fill=red!4](-2.3, -.5) rectangle (2.3, -1.8); \node[color = red] at (0,-.3) {};
\node[state, color=red] (c) [below left = of center]{$.4$};
\node (clab) [below left = .5cm and .5cm of c]{$\YY{2}$};
\node[state, fill] (a) [above = of center]{$.5$};
\node (alab) [above = .7cm of a]{$\YY{0}$};
\node[state, color=red] (b) [below right = of center]{$.6$};
\node (blab) [below right = .5cm and .5cm of b]{$\YY{1}$};
\path (a) edge node[above right] {$1$} (b);
\path (b) edge[color=red] node[below] {$\frac{3}{5}$} (c);
\path (c) edge node[above left] {$0$} (a);
\end{tikzpicture}}\right) + .3 \left(\scalebox{.55}{\begin {tikzpicture}[-latex, auto,node distance =1cm and 1.33 cm ,on grid,
semithick, baseline, anchor=base, state/.style ={ circle, draw, minimum width =.8 cm}]
\node (center) {};
\node[state, color=red] (c) [below left = of center]{$\frac{2}{7}$};
\node (clab) [below left = .5cm and .5cm of c]{$\YY{2}$};
\node[state, color=red] (a) [above = of center]{$\frac{5}{7}$};
\node (alab) [above = .7cm of a]{$\YY{0}$};
\node[state, fill] (b) [below right = of center]{$.6$};
\node (blab) [below right = .5cm and .5cm of b]{$\YY{1}$};
\path (a) edge node[above right] {$0$} (b);
\path (b) edge node[below] {$1$} (c);
\path (c) edge[color=red] node[above left] {$\frac{2}{7}$} (a);
\begin{pgfonlayer}{bg}
%\filldraw[rotate around={56.5 : ( $ (a)!0.5!(c) $ )}, color = red, fill=red!4]
%($(a)!0.5!(c) - (2.2, .55)$) rectangle ($(a)!0.5!(c) + (2.2, .7)$);
\filldraw[color = red, fill=red!4]($(a) + (-.8, 0) + .5*(a)-.5*(c)$) --($(a) + (.7, 0) + .5*(a)-.5*(c)$) -- ($(c) + (.7, 0)- .4*(a)+.4*(c)$) -- ($(c) + (-.8, 0) - .4*(a)+.4*(c)$) -- cycle;
\end{pgfonlayer}
\node[color = red] at ( $ (a)!0.5!(c) + (-1, .5)$ ) {};
\end{tikzpicture}}\right) + 
.2 \left(\scalebox{.55}{\begin {tikzpicture}[-latex, auto,node distance =1cm and 1.33 cm ,on grid,
semithick, baseline, anchor=base, state/.style ={ circle, draw, minimum width =.7 cm}]
\node (center) {};
\node[state, fill] (c) [below left = of center]{$.4$};
\node (clab) [below left = .5cm and .5cm of c]{$\YY{2}$};
\node[state, color=red] (a) [above = of center]{$\frac{5}{8}$};
\node (alab) [above = .7cm of a]{$\YY{0}$};
\node[state, color=red] (b) [below right = of center]{$\frac{3}{8}$};
\node (blab) [below right = .5cm and .5cm of b]{$\YY{1}$};
\path (a) edge[color=red] node[above right] {$\frac{5}{8}$} (b);
\path (b) edge node[below] {$0$} (c);
\path (c) edge node[above left] {$1$} (a);
\begin{pgfonlayer}{bg}
%\filldraw[rotate around={56.5 : ( $ (a)!0.5!(c) $ )}, color = red, fill=red!4]
%($(a)!0.5!(c) - (2.2, .55)$) rectangle ($(a)!0.5!(c) + (2.2, .7)$);
\filldraw[color = red, fill=red!4]($(a) + (-.7, 0) + .5*(a)-.5*(b)$) --($(a) + (.8, 0) + .5*(a)-.5*(b)$) -- ($(b) + (.8, 0)- .4*(a)+.4*(b)$) -- ($(b) + (-.7, 0) - .4*(a)+.4*(b)$) -- cycle;
\end{pgfonlayer}
\node[color = red] at ( $ (a)!0.5!(b) + (1, .5)$ ) {};
\end{tikzpicture}}\right)
\end{align*}
\caption{A demonstration of the inductive step in the proof for Lemma~\ref{lem:APref to Aprob}. The weights on the LHS are the aggregate probabilities $(y^{(i)}, y^{(j)})$ that we wish to generate, while the numbers within each vertex $y^{(i)}$ specify $p_i$. The weights of the graphs on the RHS are given by Equation~\ref{eq: cases in decomp}, with adjusted (re-normalized) probabilities $p^{[k]}$ specified within the vertices. Three subgraphs are highlighted in red, which represent the smaller sets of labels which can be decomposed according to the inductive hypothesis.}
\label{fig:eg_to_lop_decomp}
\end{figure*}
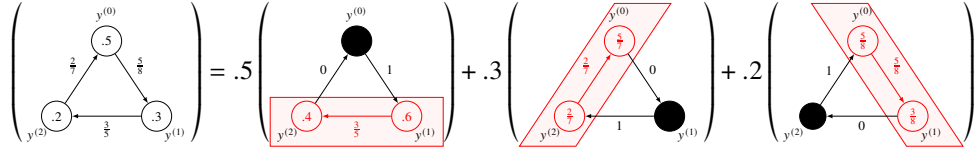

%% file: arxiv.bbl
\begin{thebibliography}{21}
\providecommand{\natexlab}[1]{#1}
\providecommand{\url}[1]{\texttt{#1}}
\expandafter\ifx\csname urlstyle\endcsname\relax
  \providecommand{\doi}[1]{doi: #1}\else
  \providecommand{\doi}{doi: \begingroup \urlstyle{rm}\Url}\fi

\bibitem[Alon(2002)]{alon2002voting}
Noga Alon.
\newblock Voting paradoxes and digraphs realizations.
\newblock \emph{Advances in Applied Mathematics}, 29\penalty0 (1):\penalty0
  126--135, 2002.

\bibitem[Bareinboim and Tian(2015)]{bareinboim2015recovering}
Elias Bareinboim and Jin Tian.
\newblock Recovering causal effects from selection bias.
\newblock In \emph{Proceedings of the AAAI Conference on Artificial
  Intelligence}, volume~29, 2015.

\bibitem[Castanedo(2013)]{castanedo2013review}
Federico Castanedo.
\newblock A review of data fusion techniques.
\newblock \emph{The scientific world journal}, 2013.

\bibitem[Cole and Hern{\'a}n(2008)]{cole2008constructing}
Stephen~R Cole and Miguel~A Hern{\'a}n.
\newblock Constructing inverse probability weights for marginal structural
  models.
\newblock \emph{American journal of epidemiology}, 168\penalty0 (6):\penalty0
  656--664, 2008.

\bibitem[Conklin and Sutherland(1923)]{conklin1923comparison}
Edmund~S Conklin and John~W Sutherland.
\newblock A comparison of the scale of values method with the order-of-merit
  method.
\newblock \emph{Journal of Experimental Psychology}, 6\penalty0 (1):\penalty0
  44, 1923.

\bibitem[Correa et~al.(2019)Correa, Tian, and
  Bareinboim]{correa2019identification}
Juan~D Correa, Jin Tian, and Elias Bareinboim.
\newblock Identification of causal effects in the presence of selection bias.
\newblock In \emph{Proceedings of the AAAI Conference on Artificial
  Intelligence}, volume~33, pages 2744--2751, 2019.

\bibitem[Fishburn(1992)]{fishburn1992induced}
Peter~C Fishburn.
\newblock Induced binary probabilities and the linear ordering polytope: A
  status report.
\newblock \emph{Mathematical Social Sciences}, 23\penalty0 (1):\penalty0
  67--80, 1992.

\bibitem[Gr{\"u}nbaum et~al.(1967)Gr{\"u}nbaum, Klee, Perles, and
  Shephard]{grunbaum1967convex}
Branko Gr{\"u}nbaum, Victor Klee, Micha~A Perles, and Geoffrey~Colin Shephard.
\newblock \emph{Convex polytopes}, volume~16.
\newblock Springer, 1967.

\bibitem[Hernán et~al.(2011)Hernán, Clayton, and Keiding]{hernanSimpsons}
Miguel~A Hernán, David Clayton, and Niels Keiding.
\newblock {The Simpson's paradox unraveled}.
\newblock \emph{International Journal of Epidemiology}, 40\penalty0
  (3):\penalty0 780--785, 03 2011.
\newblock ISSN 0300-5771.
\newblock \doi{10.1093/ije/dyr041}.
\newblock URL \url{https://doi.org/10.1093/ije/dyr041}.

\bibitem[Hernán~MA(2020)]{hernan2020whatif}
Robins~JM Hernán~MA.
\newblock \emph{Causal Inference: What If}.
\newblock Boca Raton: Chapman \& Hall/CRC, 2020.

\bibitem[Imbens and Rubin(2015)]{imbens2015causal}
G.~W. Imbens and D.~B. Rubin.
\newblock \emph{Causal inference in statistics, social, and biomedical
  sciences}.
\newblock Cambridge University Press, 2015.

\bibitem[Mazaheri et~al.(2021)Mazaheri, Jain, and
  Bruck]{mazaheri2021synthesizing}
Bijan Mazaheri, Siddharth Jain, and Jehoshua Bruck.
\newblock \href{https://arxiv.org/abs/2107.07054}{Synthesizing New Expertise
  via Collaboration}.
\newblock In \emph{2021 IEEE International Symposium on Information Theory
  (ISIT)}, pages 2447--2452, 2021.
\newblock \doi{10.1109/ISIT45174.2021.9517822}.

\bibitem[Nicolas et~al.(1785)]{nicolas1785essai}
Jean~Antoine Nicolas et~al.
\newblock \emph{Essai sur l'application de l'analyse {\`a} la probabilit{\'e}
  des decisions rendues {\`a} la pluralit{\'e} des voix. Par m. le marquis de
  Condorcet,...}
\newblock de l'Imprimerie Royale, 1785.

\bibitem[Pearl(2009)]{pearl2009causality}
Judea Pearl.
\newblock \emph{Causality}.
\newblock Cambridge university press, 2009.

\bibitem[Pearl(2022)]{pearl2022comment}
Judea Pearl.
\newblock Comment: understanding simpson’s paradox.
\newblock In \emph{Probabilistic and causal inference: The works of judea
  Pearl}, pages 399--412. 2022.

\bibitem[Peters et~al.(2017)Peters, Janzing, and
  Sch{\"o}lkopf]{peters2017elements}
Jonas Peters, Dominik Janzing, and Bernhard Sch{\"o}lkopf.
\newblock \emph{Elements of causal inference: foundations and learning
  algorithms}.
\newblock The MIT Press, 2017.

\bibitem[Robins et~al.(2009)Robins, Hern{\'a}n, Fitzmaurice, Davidian, Verbeke,
  and Molenberghs]{robins2009longitudinal}
JM~Robins, MA~Hern{\'a}n, G~Fitzmaurice, M~Davidian, G~Verbeke, and
  G~Molenberghs.
\newblock Longitudinal data analysis.
\newblock \emph{Handbooks of modern statistical methods}, pages 553--599, 2009.

\bibitem[Schweikert et~al.(2008)Schweikert, R{\"a}tsch, Widmer, and
  Sch{\"o}lkopf]{schweikert2008empirical}
Gabriele Schweikert, Gunnar R{\"a}tsch, Christian Widmer, and Bernhard
  Sch{\"o}lkopf.
\newblock An empirical analysis of domain adaptation algorithms for genomic
  sequence analysis.
\newblock \emph{Advances in neural information processing systems}, 21, 2008.

\bibitem[Shimodaira(2000)]{shimodaira2000improving}
Hidetoshi Shimodaira.
\newblock Improving predictive inference under covariate shift by weighting the
  log-likelihood function.
\newblock \emph{Journal of statistical planning and inference}, 90\penalty0
  (2):\penalty0 227--244, 2000.

\bibitem[Suissa(2008)]{suissa2008immortal}
Samy Suissa.
\newblock Immortal time bias in pharmacoepidemiology.
\newblock \emph{American journal of epidemiology}, 167\penalty0 (4):\penalty0
  492--499, 2008.

\bibitem[Tomczak et~al.(2015)Tomczak, Czerwi{\'n}ska, and
  Wiznerowicz]{tomczak2015review}
Katarzyna Tomczak, Patrycja Czerwi{\'n}ska, and Maciej Wiznerowicz.
\newblock Review the cancer genome atlas (tcga): an immeasurable source of
  knowledge.
\newblock \emph{Contemporary Oncology/Wsp{\'o}{\l}czesna Onkologia},
  2015\penalty0 (1):\penalty0 68--77, 2015.

\end{thebibliography}
